\documentclass{article}

\usepackage{graphicx}

\usepackage[utf8]{inputenc}
\usepackage[utf8]{inputenc}
\usepackage[english]{babel}
\usepackage{amsthm}
\usepackage{amsmath}
\usepackage{amsfonts}
\usepackage{booktabs}
 
 \usepackage{hyperref}

 \usepackage{iclr2018_conference,times}

\newtheorem{theorem}{Theorem}
\newtheorem*{theorem*}{Theorem}
\newtheorem{lemma}{Lemma}
\newtheorem{definition}{Definition}

\newcommand{\M}{\mathcal{M}}
\newcommand{\E}{\mathbb{E}}
\newcommand{\N}{\mathbb{N}}
\newcommand{\R}{\mathbb{R}}
\newcommand{\D}{\mathcal{D}}
\newcommand{\A}{\mathcal{A}}

\newcommand{\pth}[1]{\left( #1 \right) }
\newcommand{\abs}[1]{{\left| #1 \right| }}
\newcommand{\set}[1]{\left\{ #1 \right\} }
\newcommand{\cmnt}[1]{\ignorespaces}
 
\iclrfinalcopy

\title{Train on Validation: \\ Squeezing the Data Lemon}

\author{Guy Tennenholtz\\
	Technion - Israel Institute of Technology
\\
\AND{Tom Zahavy} \\
Technion - Israel Institute of Technology \\
\AND{Shie Mannor} \\
Technion - Israel Institute of Technology
}

\begin{document}



\interfootnotelinepenalty=10000
\maketitle
\begin{abstract} 
Model selection on validation data is an essential step in machine learning. While the mixing of data between training and validation is considered taboo, practitioners often violate it to increase performance. Here, we offer a simple, practical method for using the validation set for training, which allows for a continuous, controlled trade-off between performance and overfitting of model selection. We define the notion of on-average-validation-stable algorithms as one in which using small portions of validation data for training does not overfit the model selection process. We then prove that stable algorithms are also validation stable. \cmnt{Finally, we provide experimental evidence of increased test performance, using our method, with a minor trade-off in selection bias for the MNIST and CIFAR-10 datasets on stable algorithms as well as state-of-the-art in neural networks, showing promising results.} Finally, we demonstrate our method on the MNIST and CIFAR-10 datasets using stable algorithms as well as state-of-the-art neural networks. Our results show significant increase in test performance with a minor trade-off in bias admitted to the model selection process.
\end{abstract}

\section{Introduction}

Machine learning and statistical analysis play a key role in the development of artificial intelligence.
\cmnt{Supervised learning has been a great success in real-world applications, including financial applications, algorithmic trading, biology, energy, and pattern recognition applications for speech and images. Learning constitutes two steps: training a model, and model selection.}
Model selection is the task of selecting a statistical model from a set of candidate models, given data, and is usually done by dividing the data into two sets: training and validation. The validation set is the set of examples used for model selection. The validation set provides an unbiased evaluation of each model fit on the training set, used to compare performance of the different models and decide which one to use (e.g., choosing the number of hidden layers in a neural network). 

We take particular interest in the regime of \textbf{``medium" sized datasets} in which applying learning algorithms is practical, but additional data can boost performance. In such datasets the partitioning of the data can be harmful to overall performance. In order to overcome this problem, we propose using a ``transparent" validation set in which some information can ``leak out" during training (i.e., by training on a portion of the validation data). Training on some of the validation examples in a \textit{controlled} manner, enables the increase of performance with a reasonable and constrained trade-off in the bias of model selection. While such a use of validation is thought of as taboo, we show that this restriction can be relaxed under certain conditions. We claim that training on validation data should not trivially be considered bad practice, and one should carefully use the information in this set for fitting a model, with an attempt to ``squeeze" the most out of the available data, whenever possible.

Ideas for selecting models with the reuse of data have been studied in various fields of research \cite{refaeilzadeh2009cross, cawley2010over, dwork2015reusable, dwork2017guilt, kohavi1995study}. 
A well known, admissible practice known as $K$-Fold Cross-Validation \cite{stone1974cross, geisser1975predictive, schaffer1993selecting, refaeilzadeh2009cross} attempts to reduce the amount of bias caused by a particular choice of validation set. Another method, known as Bootstrapping \cite{efron1982jackknife}, samples the dataset (with replacement), validating on the unsampled results. These methods, however, do not solve the problem of the dichotomic separation of data (that is, both methods could potentially use more of the dataset during the training process). There exist common protocols that attempt to overcome this problem. The most prevalent method that is used in practice is to train on the whole dataset (train and validation) \textit{after} a model has been selected. It has been demonstrated that such an incorrect use of validation data for model selection can lead to a misleading optimistic bias in performance evaluation \cite{cawley2010over}, resulting in an unreliable choice of model. 

The use of validation data for training has a potential for great bias in estimated results, but can be partially done when regularized appropriately \cite{cawley2010over, dwork2017guilt}. In this paper, we show there exists a complete spectrum between not using any of the validation data and using all of the validation data for training. When data is indispensable, this calibration may increase the performance of learning algorithms with a minor trade-off in selection bias. In our procedure, we sample each example from the validation set with probability $p \in [0,1]$ and add it to the training set. Once training is complete, the validation set is used for model selection (including examples that were used for training). In Theorem \ref{thm:MainResult}, we prove that using this procedure the correct model can still be chosen if the learning algorithm is stable and $p \sim \frac{1}{|V|}$, where here $V$ is the chosen validation set. Our results show that when $p$ is small enough, a larger, augmented training set can be used, while simultaneously selecting a model using the partially seen validation set. 
In practice we observed that for iterative algorithms, such as SGD, it is possible to deploy our sampling procedure at each iteration, allowing effective use of the whole dataset for training. Our procedures establish a form of regularization to the validation examples' effect on the training process.

Previous work has shown that using validation data for model selection in an adaptive, multiple step process, may bring undesirable ramifications. ``Freedman's Paradox" \cite{freedman1983note} depicts the potentially dangerous scenarios where results of one statistical procedure performed on the data are fed into another procedure performed on the same data. Freedman observed, through simulation, that when the number of variables is of the same order as the number of data points, insignificant variables can seemingly become significant.  It has been shown that the reuse of the validation set adaptively can greatly increase the risk of spurious inference \cite{dwork2015reusable, dwork2017guilt}. Even stable algorithms, when running sequentially, may result in a procedure that is unstable. Contrary to our work, these studies consider an adaptive process of reusing the validation set, whereas our method selects a model only once after training is complete (i.e., no further use of the data is performed after training). Such adaptive validation procedures, while complementary to our proposed method, must be taken into consideration and performed with extra care. An exhaustive review of these methods can be found in the related work section, and a summary can be found in Table~\ref{tb:comparison}.

Our proposed sampling method is the first principled study of how to use validation data \textit{for training} under theoretical guarantees. In addition, we show experimental evidence that our sampling method can be applied to a broad class of algorithms, including well known stable algorithms such as linear regression, k-Nearest Neighbors \cite{altman1992introduction}, and Support Vector Machines \cite{hearst1998support, shalev2014understanding}. In addition, we conduct experiments using our method on state-of-the-art in neural networks \cite{huang2016densely}, showing $1.1\%$ (absolute) increase in performance on the CIFAR-10 \cite{krizhevsky2009learning} dataset for image classification.

\begin{table*}[t!]
\caption{Comparison of different protocols to our method. Methods which are adaptive have a multiple step procedure which uses previously chosen validation accuracies for following stages. Methods which are marked as complementary can be used in conjunction with our method. Methods which are marked as enlarged are ones which effectively expand the initial training set, except for bootstrapping which trains on copies of identical examples. Finally, methods which are marked as potentially biased are likely to obtain incorrect, biased performance estimates. See the related work section for an exhaustive review of these methods. }

\label{tb:comparison}
\vskip 0.15in
\begin{center}
\begin{small}
\begin{sc}
\begin{tabular}{lcccc}
\toprule
Method & Adaptive & Complementary & Enlarged & Potentially \\
 & Evaluation &  & Training Set & Biased \\
\midrule
k-Fold CV \cite{refaeilzadeh2009cross}
& - & $\surd$ & $\times$ & $\times$ \\
Bootstrap \cite{efron1982jackknife}
& $\times$ & $\surd$ & $\surd$ & $\times$\\
Thresholdout \cite{dwork2017guilt}
& $\surd$ & $\surd$ & $\times$ & $\times$ \\
Train after Test
& - & $\times$ & $\surd$ & $\surd$\\
Our method
& $\times$ & - & $\surd$ & Depends on $p$ \\
\bottomrule
\end{tabular}
\end{sc}
\end{small}
\end{center}
\vskip -0.1in
\end{table*}

\section{Preliminaries}

We consider the classical supervised learning setting, where we wish to predict a random variable $Y$ based on observations from another random variable $X$. Let $S$ be a dataset of $\abs{S}$ random pairs $\set{z_i = (x_i, y_i)}_{i=1}^\abs{S}$ drawn independently from a fixed unknown distribution $\D$ on ${\mathcal{Z} = \mathcal{X}\times\mathcal{Y}}$. A learning algorithm ${\A:\mathcal{Z}^\abs{S}\to H}$ is a mapping from $\mathcal{Z}^\abs{S}$ to a hypothesis class $H$. We measure the quality of predictions using a loss function ${\ell:H\times\mathcal{Z} \to \R_+}$. That is, $\ell(h,z)$ measures the loss of predicting example $z = (x,y)$ using hypothesis $h$. The risk associated with hypothesis $h$ is then defined as the expectation of the loss function:
${
R_\D(h) = L_\D(h) = \E_{z \sim \D} \ell(h,z).
}
$

In general, the risk $R_\D$ cannot be computed because the distribution $\D$ is unknown to the learning algorithm. However, we can compute an approximation, called the empirical risk, by averaging the loss function over the set $S$. We thus define the empirical risk by
${
R_S(h) = L_S(h) = \frac{1}{\abs{S}} \sum_{i=1}^\abs{S} \ell(h, z_i).
}
$
The Empirical Risk Minimization Principle (ERM) states that a learning algorithm should choose a hypothesis $h$ which minimizes the empirical risk $R_S(h)$.

\subsection{Model Selection}
It is often essential to compare multiple learning algorithms to determine which one works best on the problem at hand. It is thus a common practice to define part of the dataset as a validation (holdout) set. The goal of using a validation set is to provide an unbiased evaluation of a model fit on the training dataset while selecting the best model. 


We assume the dataset $S$ is divided into a training set $T$ and validation set $V$ (i.e., $S = T \cup V$), and define the empirical loss and validation loss by
$L_T(h) = \frac{1}{\abs{T}} \sum_{z_i \in T} \ell(h, z_i)$ and
$L_V(h) = \frac{1}{\abs{V}} \sum_{z_i \in V} \ell(h, z_i)$, respectively. The former is used as the minimization criterion and the latter as the criterion for model selection. \cmnt{That is, we will train according to $L_T(h)$ and choose the best model according to $L_V(h)$.}


In the finite case, let $\set{\M_i}_{i=1}^K$ be a set possible models, where $K \in \N$, and let $L_V^{(i)}(\A(T))$ be the validation loss of algorithm $\A$ using model $\M_i$ when training on $T$. We denote by $\M_{i^*}$ the optimal model choice, where
\begin{equation}
\label{eq:bestChoice}
i^* \in \arg\min_i L_V^{(i)}(\A(T)).
\end{equation}

\subsection{Stability}
A stable algorithm is one in which small change to the input results in a small change to the loss. One form of stability measures the effect of replacement on the output. As we show in Theorem \ref{thm:MainResult} and the Experiments Section, this property has a major effect on our ability to train on validation examples. 

Let $T = \set{z_1, z_2, \hdots, z_n}$ be a dataset of $n$ examples sampled independently from an unknown distribution $\D$. We would expect that for a stable algorithm $\A$, the loss $\ell(\A(T),z_i)$ would be close to $l\left(\A\left(T \backslash \{z_i\} \cup \{z\}\right), z_i\right)$, where $z \sim \D$. The On-Average-Replace-One-Stability (OAROS) criterion proposed in \cite{shalev2014understanding}
and recalled in Definition \ref{def:OAROS} below measures in expectation this effect on the loss by uniformly replacing one example from the training set.

\begin{definition}[On-Average-Replace-One-Stability]
\label{def:OAROS}
Let  ${\epsilon: \N \to \R_+}$ be a monotonically decreasing function. 
We say that a learning algorithm $\A$ is On-Average-Replace-One-Stable with rate $\epsilon(n)$ w.r.t. a loss function $\ell$, if for every distribution $\D$, and training sample $T$ of $n$ data points from $\D$, we have:
\begin{align*}
\E_T \E_{x'\sim \D}\E_{z \sim U(T)} \left| \ell(\A(T), z) - \ell(\A(T \cup \{x'\} \backslash \{z\}), z) \right| \\
\leq \epsilon(n),
\end{align*}
where here $U(T)$ is the uniform distribution on the elements in $T$.
\end{definition}

\noindent Other forms of stability which are closely related to OAROS include uniform stability \cite{bousquet2002stability} and Leave-One-Out stability \cite{shalev2009learnability, elisseeff2003leave}. Intuitively, stability measures the sensitivity to perturbations in the training set. In particular, it is known that stability of the ERM is sufficient for learnability, as stated in the following theorem. It has even be argued that stability is a necessary condition for learnability \cite{mukherjee2006learning}.
\begin{theorem*}
Let $\A$ be an ERM algorithm with OAROS rate $\epsilon(n)$, then
\begin{equation*}
\E_{S \sim \D^m} \left[ L_{\D}(\A(S)) - L_S(\A(S)) \right] \leq \epsilon(n).
\end{equation*}
\end{theorem*}

Note that stability is not a binary property of an algorithm. In fact, the stability rate $\epsilon(n)$ offers a measure of how stable an algorithm is. As we show in the experiment section, stable algorithms are often more stable in practice (i.e., the bounds on known stability rates are looser in practice).

\section{Main Result}

In this section, we show that the requirement of a segregated validation set can be relaxed under stability assumptions of the learning algorithm. Our main goal is to use validation data for training, limiting the amount of bias presented to the model selection stage. One way of doing this is by defining a form of stability over validation examples when selecting the model. \cmnt{More specifically, our goal is to find algorithms in which the use of a validation example for training would have a bounded and decreasing effect on model selection.}

Let $\mathcal{P}$ be a procedure that receives as input the training set $T$ and the validation set $V$, and outputs an augmented set $\mathcal{P}(T, V) \subseteq T \cup V$.
We wish to choose a procedure $\mathcal{P}$ such that Equation (\ref{eq:bestChoice}) would still hold. Specifically,
\begin{equation}
\label{eq:sameChoice}
\arg\min_i L_V^{(i)}(\A(\mathcal{P}(T, V))) = \arg\min_i L_V^{(i)}(\A(T)).
\end{equation}

To ensure this holds with high probability we choose a procedure for which the ordering of $\{L_V^{(i)}(\A(T))\}$ and $\{L_V^{(i)}(\A(\mathcal{P}(T, V)))\}$ are the same. That is, $\forall i, j$ if $L_V^{(i)}(\A(T)) \geq L_V^{(j)}(\A(T))$, then ${L_V^{(i)}(\A(\mathcal{P}(T, V))) \geq L_V^{(j)}(\A(\mathcal{P}(T, V)))}$. Alternatively, we can choose to randomly sample the smallest possible unit of data (i.e., one example) from $V$ and add it to $T$. Our goal then is to ensure that for any random choice of example passed to the training procedure the loss over the validation examples would not change by much when selecting a model over the \textit{full validation set} (including the sampled example). This would thereby ensure the correct choice of a model. We define a validation-stable algorithm to be one in which such an addition results in small change to the output. The On-Average-Validation-Stability (OAVS) criterion measures in expectation the effect of uniformly choosing one example from $V$ and adding it to $T$ on the loss.

\begin{definition}[On Average Validation Stability]
Let  ${\epsilon: \N \times \N \to \R_+}$ be a monotonically decreasing function in both parameters separately. 
We say that a learning algorithm $\A$ is On-Average-Validation-Stable with rate $\epsilon(n,m)$ if for every distribution $\D$, training sample $T$ of $n$ data points from $\D$, and validation sample $V$ of $m$ data points, we have:
\begin{align*}
\E_{T,V} \E_{y,y' \sim U(V)} \left[ \ell(\A(T), y) - \ell(\A(T \cup \{ y'\}), y)\right] \\
\leq \epsilon(n,m),
\end{align*}
where here $U(V)$ is the uniform distribution on the elements in $V$.
\end{definition}

Indeed, if an algorithm $\A$ is OAVS stable with $\epsilon(n,m)$, and 
\begin{equation}
\label{eq:epsilonCondition}
\epsilon(n,m) < \min_{i \neq j} \abs{L_V^{(i)}(\A(T)) - L_V^{(j)}(\A(T))},
\end{equation}
sampling one unit of data from $V$ will in expectation satisfy Equation (\ref{eq:sameChoice}). Alternatively, since $\epsilon(n,m)$ is monotonically decreasing in $n$ and $m$, then for a satisfactory large training and validation set, condition (\ref{eq:epsilonCondition}) would hold, and thereby condition (\ref{eq:sameChoice}) would hold as well with high probability (by use of Chernoff bound, see the Appendix).

Validation-stability is an important property of a learning algorithm and bears a close relation to the stability of an algorithm. In the next theorem we show that OAROS algorithms are also OAVS (under standard convergence assumptions). A proof can be found in the Appendix.
\begin{theorem}
\label{thm:MainResult}
Let ${T,V \sim \D}$ and denote \mbox{$\tilde{T} = T \cup \{x'\}$} where $x' \sim \D$. Denote $|T| = n, |V| = m$. Let $\A$ be an ERM algorithm which is OAROS stable with $\epsilon_1(n)$ and satisfies
\begin{equation*}
\label{eq:ERM_convergence}
\E_{\tilde{T}} \left[ L_T(\A(T)) - L_{\tilde{T}}(\A(\tilde{T})) \right] \leq \epsilon_2(n),
\end{equation*}
where $\epsilon_2(n)$ is some monotonically decreasing function in $n$. Then it holds that $\A$ is On-Average-Validation-Stable with rate $\epsilon(n,m) = \pth{3 + \frac{1}{m}}\epsilon_1(n) + \epsilon_2(n).$
\end{theorem}

Theorem \ref{thm:MainResult} gives a fundamental notion of when validation examples can be used for training and to what extent. In the following section we describe several possible procedures that use validation examples and show evidence of their ability to properly select the best model, with an increase to overall performance.

\section{Experiments}

In this section, we describe two sampling procedures for training using a validation set. We then show results of using such procedures on stable algorithms as well as state-of-the-art in neural networks for image classification. Both procedures constitute of a control parameter ${p \in [0,1]}$ which restricts the amount of validation data that is used for training. 

\textbf{The presample procedure} samples examples before training starts. Every example from $V$ is sampled independently with probability $p$ and is added it to $T$.

\textbf{The batch-sample procedure} samples full batches. At every training iteration, a batch is either taken from $T$ or $V$. Specifically, with probability $p$ a batch $B_V$ is used from $V$, and with probability $1-p$ a batch $B_T$ is used from $T$, independent of the other batches.

\cmnt{\textbf{The online-sample procedure} combines the above methods. Each example slot in every batch is sampled from $V$ with probability $p$ and with probability $1-p$ from $T$, independent of the other examples.}

The last method samples at a per-iteration-level, allowing a uniform use of the validation set, where $p$ acts as a reguralizer on validation examples (lower $p$ = higher regularization). A value of $p$ under the batch sampling procedure is comparable to a value of $\tilde{p} = \frac{|T|}{|V|}p$ under the presampling procedure. This, in turn, means that the batch-sample procedure can effectively oversample validation examples (thereby ``seeing" more validation examples than training examples).

\begin{figure}[t!]
\centering
\includegraphics[width=\linewidth]{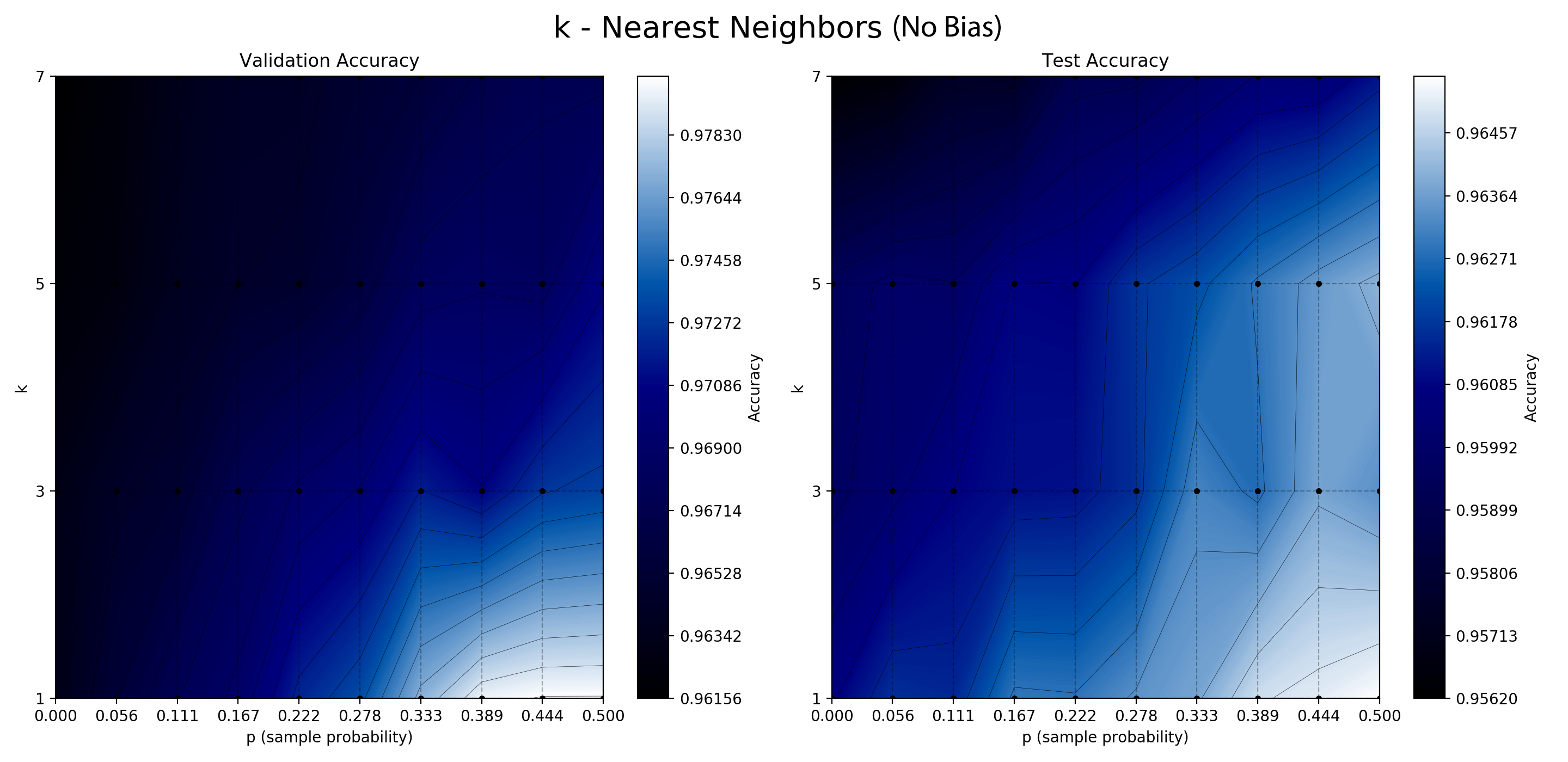}
\includegraphics[width=\linewidth]{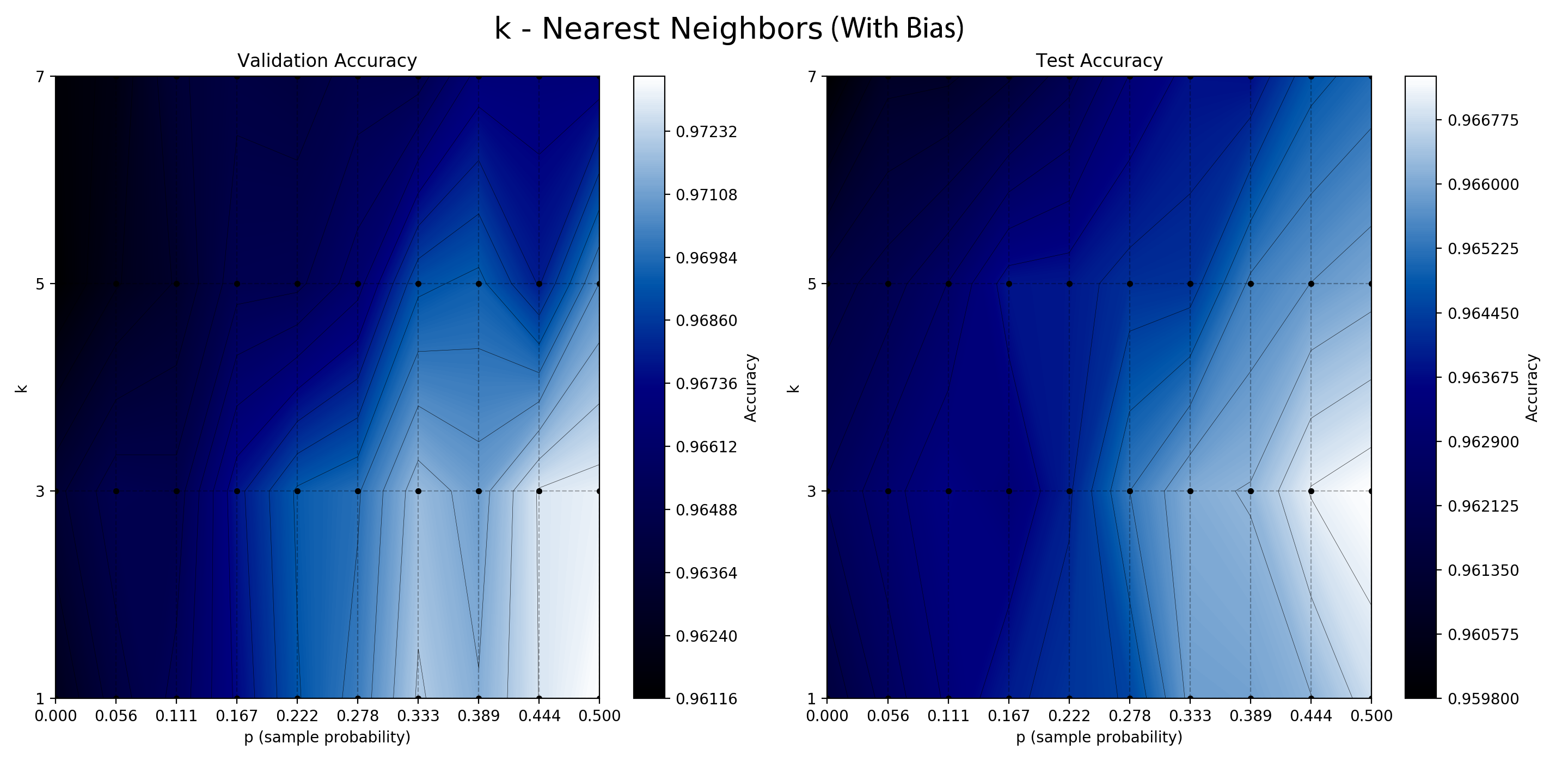}
\caption{Validation and test accuracies for the presample procedure using k-NN. 
Top graphs show results for unbiased validation and test sets, while bottom graphs show results for biased validation and test sets. }
\vspace*{-2ex}
\label{fig:knn}
\end{figure}
\vspace*{-2ex}
\subsection{Stable Algorithms}
The presample procedure was tested on linear regression, k-Nearest Neighbors (k-NN), and Support Vector Machines (SVM), which have been previously shown to be stable \cite{bousquet2002stability, devroye1979distribution, e2000study}. Results show compliance with our approach in the strongest sense. Essentially all stable algorithms that were tested did not overfit the model selection process, allowing for an absolute use of validation examples, while maintaining negligible bias in model selection. We present results for SVM and k-NN on an image classification task. Linear regression results can be found in the Appendix.

\textbf{Dataset:} For all of our experiments we used the MNIST dataset \cite{lecun1998mnist}, containing 70000, 28x28 grayscale images of handwritten digits drawn from 10 classes. 

\textbf{k-NN:} We tested k-NN with the standard Euclidean metric. The dataset was divided to train, validation and test sets with 50000 examples used for training and 10000 examples used for selecting an optimal value of $k$, with values between 1 and 7. Figure \ref{fig:knn} (top) shows accuracy results on the validation and test sets for different values of $k$. Colors of the contour map range from black to blue to white, with black depicting lower accuracy scores, and white depicting higher accuracies. The x-axis of the map shows different $p$ values with $p$ ranging from $0$ to $0.5$. Every point on the grid shows a result of an experiment with a given $p$ and $k$. Test results show a monotone increase in performance for increased values of $p$. Model selection was found to be independent of $p$, allowing for the use of the full validation set for training, with an optimal choice of $k = 1$ giving a $0.7\%$ increase in overall performance.

\textbf{SVM:} We used the radial basis function (RBF) kernel $\exp \left( -\gamma || x_i - x_j || \right)$. The dataset was divided to train, validation and test sets with the validation set used to choose an optimal value of $\gamma$. Figure \ref{fig:svm} (top) shows accuracy results on the validation and test sets for different values of $\gamma$. Each line on the test accuracy plot depicts accuracies for different values of $\gamma$. We noted no change in model selection for all values of $p$. Test accuracies were only slightly improved with less than $0.1\%$ increase in overall performance. 

\textbf{k-NN (Bias Experiment):} An additional benefit of using validation data for training is for dataset shift. Dataset shift is a common problem that occurs when the joint distribution of inputs and outputs differs between training and test stages \cite{quionero2009dataset}. Dataset shift is present in most practical applications with examples of user queries in search engines, utterances inputted to dialogue systems, and data on products in online stores. Such a drift in distribution may have calamitous effects on performance when using outdated data for training. Our sampling procedure, apart from effectively creating a larger dataset for training, can be used as inductive bias on the training procedure for such time-dependent data. When using a validation set of this form, the latest, most recent data sample is usually used for testing. By sampling from the validation set, we can decrease erroneous bias caused by time deformation. 

We tested the effect of biased validation and test sets on our sampling procedure by introducing artificial bias to the validation and test sets. The bias was comprised of random flipping of the images (left right and up down) as well as an increased probability to class 9 (the training set distribution was set to the uniform distribution on all classes). Figure \ref{fig:knn} (bottom) shows accuracy results using k-NN on the validation and test sets for different values of $k$. Model selection was not affected by $p$, allowing for the use of the full validation set for training, with an optimal choice of $k = 3$, and a $0.9\%$ increase in overall performance.

\textbf{SVM (Training Size Experiment):} To examine the effect of the size of the training set we rerun the SVM experiment with a smaller dataset - half in size. Figure \ref{fig:svm} (bottom) shows accuracy results on the validation and test sets for the medium sized training set. While our previous SVM results showed an indiscernible effect on test accuracy, validation sampling had an evident effect on performance on the reduced training set, with $2.3\%$ increase in test accuracy.

\begin{figure}[t!]
\centering
\includegraphics[width=\linewidth]{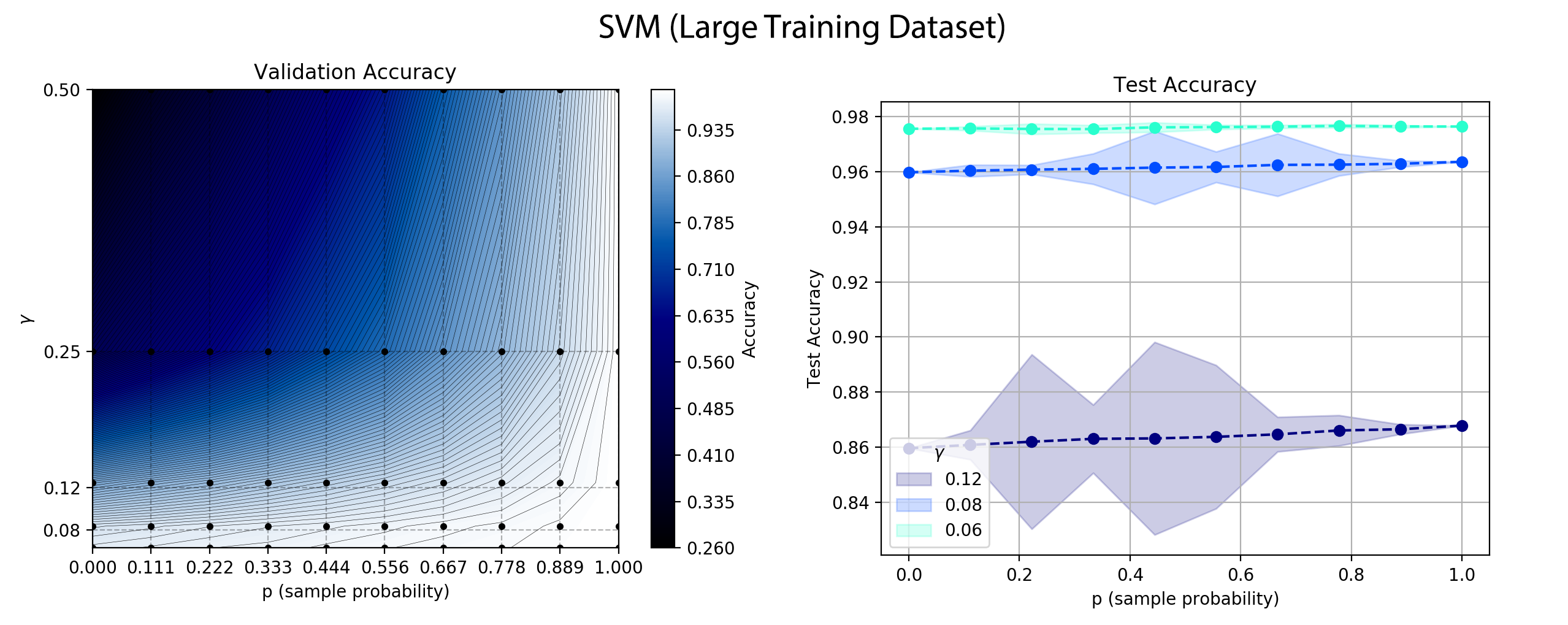}
\includegraphics[width=\linewidth]{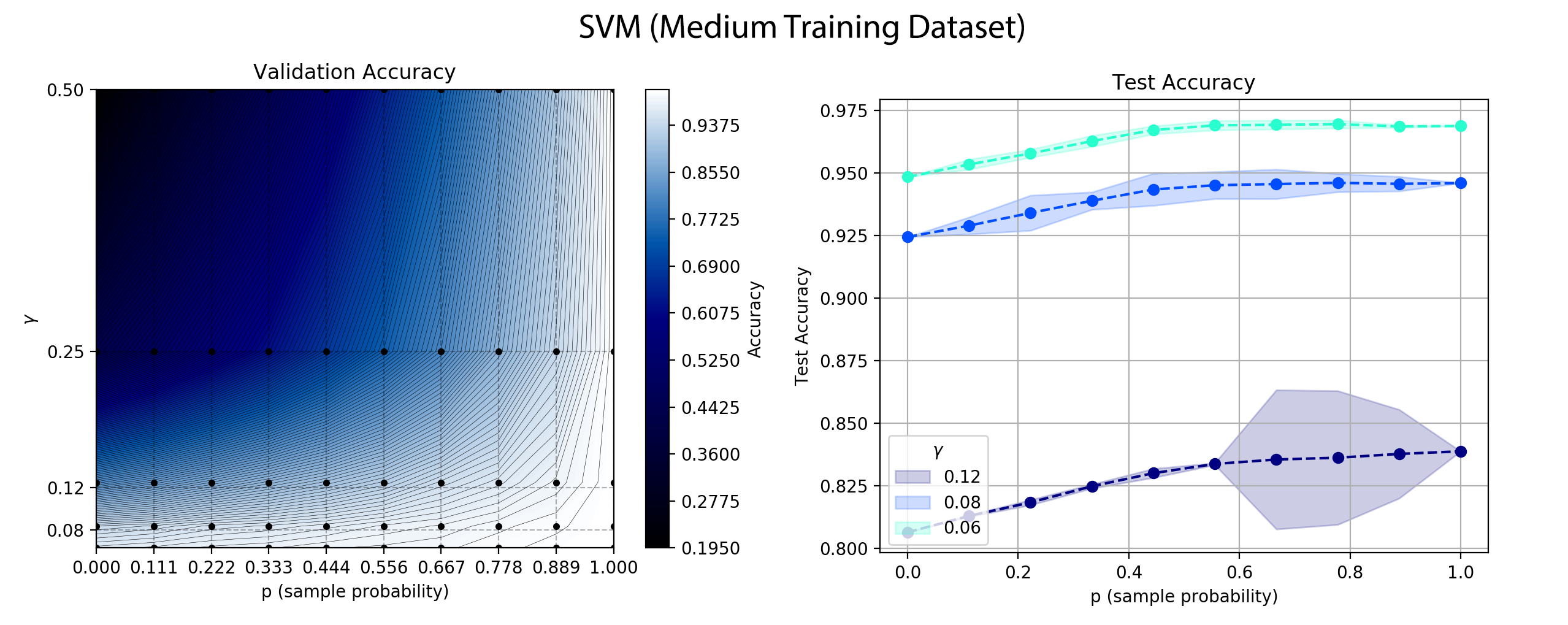}
\caption{Validation and test accuracies for the presample procedure using SVM and different choices of $\gamma$ (the kernel parameter). In the medium training set experiment (bottom), test accuracy increases for $p < 0.5$, and does not changed for $p > 0.5$. In the large training set experiment (top), test accuracy does not change. }
\label{fig:svm}
\vspace*{-2ex}
\end{figure}

\textbf{Summary:} All experiments showed an advantage in validation sampling with stable algorithms. Dataset shift bias and reduced training set size both had a clear influence on sampling efficiency. \cmnt{Stability is not a binary property of an algorithm. In fact, the stability rate $\epsilon(n)$ offers a measure of how stable an algorithm is. Stable algorithms are often more stable in practice.} In addition, all experiments performed did not overfit the model selection process. We suspect this is due to the strong stability property of these algorithms in practice. SVM classified the data using only a small number support vectors, establishing a negligible effect in test accuracy for adding new examples. For this reason, the effect of sampling was only evident when the training set was small enough. In the next section, we show that neural networks, which may have weaker stability properties, are more susceptible to overfitting the model selection process. Nevertheless, as we'll see next, our sampling procedure can be used to increase their performance significantly.

\subsection{DenseNet}

Unlike regression, SVM, and k-NN, neural networks have not yet been shown stable, yet they achieve state-of-the-art results in a variety of machine learning problems. In this section, we demonstrate our sampling procedure on Densely Connected Convolutional Networks (DenseNet) \cite{huang2016densely}. The DenseNet architecture contains dense shortcut connections between layers and was recently shown to achieve state of the art results on CIFAR-10 and 100. \cmnt{The DenseNet architecture contains shortcuts between layers. For each layer, the feature maps of all preceding layers are treated as separate inputs whereas its own feature maps are passed on as inputs to all subsequent layers. This connectivity pattern yields state-of-the-art accuracies on CIFAR-10 and CIFAR-100. }

\textbf{Setup:} The dataset we used was the CIFAR-10 dataset. The dataset contains 32x32 color images drawn from 10 classes. The dataset of size 60000 examples was divided to train, validation, and test sets with 40000 examples used for training and 15000 examples used for model selection (validation set). The data was normalized by its mean and standard deviation, and a standard augmentation was applied to the training data by randomly mirroring, padding, and shifting the images. 

All of our experiments used the vanilla Densenet model with three dense blocks of 12 layers each and a growth rate set to 12. The network was implemented using TensorFlow \cite{abadi2016tensorflow} with an SGD optimizer momentum of 0.9. Weight decay was set to $10^{-4}$. For model selection, we ran experiments on two parameters: learning rate multiplier ($LRM$) and batch size. The initial learning rate was set to $0.1*LRM$ and was lowered to $0.01*LRM$ at epoch $150$, and to $0.001*LRM$ at epoch $225$. Training was done for a total of $300$ epochs. We tested 9 values of $LRM$ ranging from $0.1$ to $8.9$ and 4 batch sizes of $16, 32, 64,$ and $128$. When a parameter was not tested it was set to its vanilla value. $LRM$ was set to a default value of $1$ and batch size to a default value of $64$.
\begin{figure*}[h!]
\centering
\includegraphics[width=0.7\linewidth]{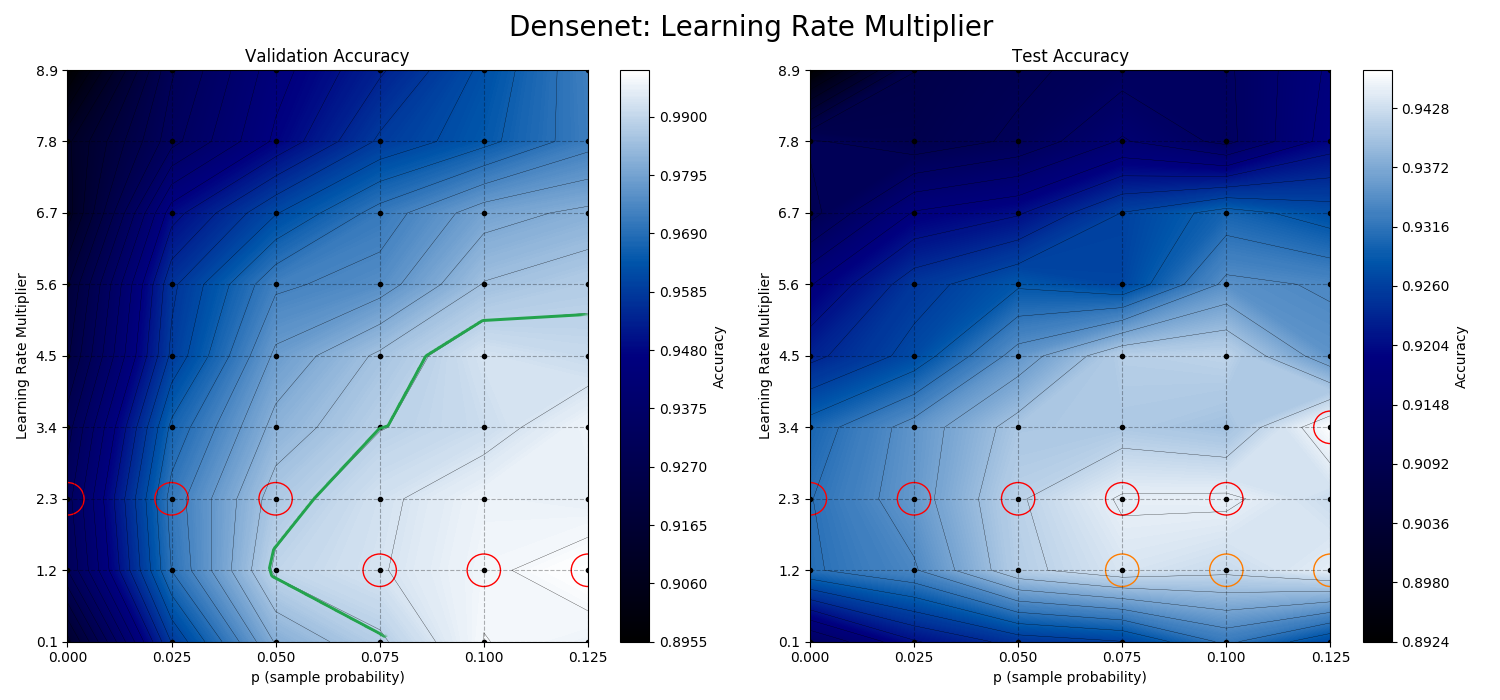}
\includegraphics[width=0.7\linewidth]{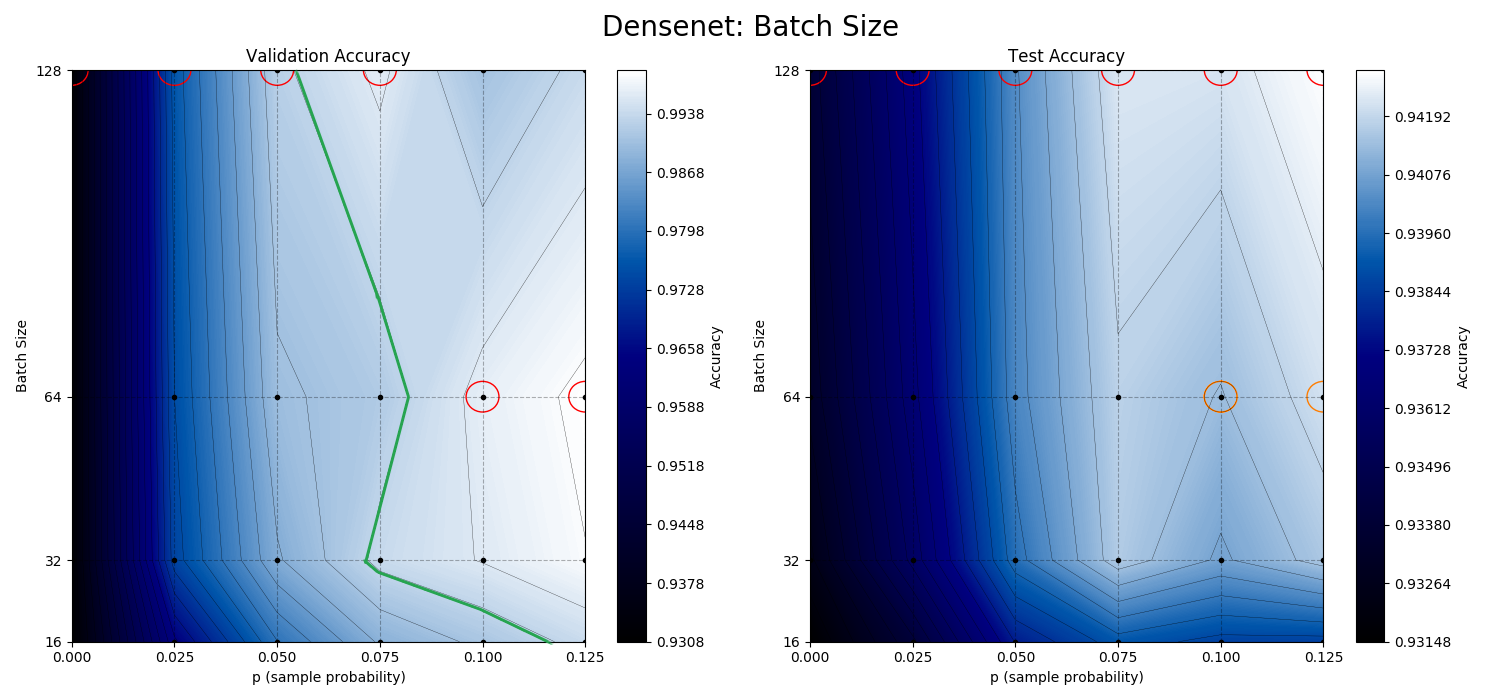}
\caption{Validation and test accuracies of the batch sampling procedure with a batch sampled with probability $p$. Top graphs show different values of learning rate multipliers ($LRM$), while bottom graphs different batch sizes. Red circles show optimal values of $LRM$ / batch size for each value of $p$. Orange circles show test accuracies for the models selected when different from best model. The green line depicts the ``knee" of the validation plot. Test performance scores to the left of the knee achieve accuracies of $94.2\%$ ($LRM$) and $94.1\%$ (batch size) when sampling with probability $p = 0.05$ compared to accuracies of $93.1\%$ ($LRM$) and $93.3\%$ (batch size) without validation sampling. Model selection receives improper bias when $p > 0.075$ ($LRM$) and $p > 0.1$ (batch size). Validation scores increase as $p$ increases. }
\label{fig:desnenet_LRM}
\end{figure*}

\textbf{Results:} The presample procedure (sample from validation w.p. $p$ before training starts) and batch-sample procedure (sample a validation batch w.p. $p$) were tested showing similar results. Figure \ref{fig:desnenet_LRM} shows accuracies of the validation and test sets for changing of the $LRM$ and batch size using the batch-sample procedure. Colors of the contour map range from black to blue to white, with black depicting lower accuracy scores and white depicting higher accuracies. The x-axis of the map shows different values of $p$, ranging from $0$ to $0.125$. \cmnt{Note that a value of $p$ for a training set of size $n$ and validation set of size $m$ under the batch sampling procedure is comparable to a value of $\tilde{p} = \frac{n}{m}p$ under the pre-sampling procedure (in our case $\tilde{p} \approx 2.67p$).}  The y-axis depicts different values of $LRM$ / batch size. Every point on the grid shows a result of an experiment with a given $p$ and $LRM$ / batch size.

For each value of $p$ the best choice of $LRM$ or batch size is marked by a circle. For the validation accuracy plots (left), the red circles mark the best model (which would be chosen). For the test accuracy plots (right), the red circles depict the (true) optimal model, while the orange circles depict the model that would have been selected (by the validation set). Thus, if for a certain value of $p$ the red and orange circles are not at the same value, overfitting of the model selection has occurred. For instance, in Figure~\ref{fig:desnenet_LRM} when $p=0.025$ the chosen model of $LRM=2.3$ achieves a validation accuracy of $96\%$ (top, left) with it also having an optimal test accuracy of $93.5\%$ (top, right). On the other hand, when $p=0.1$ the chosen model of $LRM=1.2$ achieves a validation accuracy of $99\%$ (top, left), whereas its test accuracy achieves sub-optimal accuracy (top, right). The true optimal model for the case of $p=0.1$ is for $LRM=2.3$ with a test accuracy of $94\%$.

As expected, a monotone increase in validation accuracy occurred as $p$ increased (left). This is due to the increased number of validation examples trained upon. While validation accuracy scores had a linear dependence in $p$ for all of the stable algorithms tested, DenseNet showed a strongly concave behavior. Moreover, DenseNet admitted to undeniable selection bias above a certain value of $p$. Specifically, model selection received improper bias when $p > 0.075$ ($LRM$, top right) and $p > 0.1$ (batch size, bottom right). Nevertheless, a valid value of $p$ always existed in which model selection was possible.

The sudden increase in validation accuracy for small values of $p$ reveals the necessity of the validation data. Meanwhile, small values of $p$ introduce negligible bias to the model selection process. This behaviour suggests choosing $p$ at an equilibrium point in which performance and model selection bias are balanced. Such a choice of $p$ cannot rely on test accuracies (see ``Data Dredging" in the Related Work Section). Then, as a rule of thumb, we chose to use the ``knee" of the validation plot for choosing a final value of $p$. The knee of the validation plot represents the point in which overfitting of validation examples has not occurred yet, whereas sufficient information for training has been used. We mark the knee of the validation plot in Figure~\ref{fig:desnenet_LRM} by a green line, dividing the graph into two parts: left of the knee where model selection is possible, and right of the knee where selection is prone to significant bias. Indeed, all of our experiments suggest model selection before the knee did not overfit the model selection process. Test performance scores at the knee show increased accuracies of $94.2\%$ ($LRM$) and $94.1\%$ (batch size) compared to accuracies of $93.1\%$ ($LRM$) and $93.3\%$ (batch size) without validation sampling.

\textbf{Summary:} There exists a clear advantage to sampling validation examples during training. While stable algorithms were remarkably robust in model selection (even when the full validation set was used), our experiments on DenseNet show a clear limit to the sampling magnitude. Nevertheless, both the presample and batch-sample procedures significantly improved results, with an advantage of effectively using more validation data using batch samples. 

\section{Related Work}
\label{section:relatedWork}
In this section, we discuss related methods for model selection and how our work is positioned relative to them. Table~\ref{tb:comparison} contains a summary of these methods.

\textbf{k-Fold Cross-Validation. }
In order to avoid overfitting in model selection, it is necessary to evaluate the models on a holdout (not used in training) validation set. However, by partitioning the available data, we drastically reduce the number of samples which can be used for learning the model, and the results can depend on a particular random choice for the pair of train or validation sets. One way to overcome this is by the use of $k$-fold Cross-Validation \cite{refaeilzadeh2009cross}. In $k$-fold Cross-Validation, the dataset is partitioned into $k$ equal sized sets. Of the $k$ sets, a single set is retained as the validation set, and the remaining $k-1$ sets are used as a training set. \cmnt{The cross-validation process is then repeated $k$ times, with each of the $k$ sets used exactly once as the validation set.} The $k$ results from the folds are then averaged to produce a single estimation. Leave-one-out validation is a special case of Cross-Validation, where all folds consist of a single instance. This type of validation is, of course, more expensive computationally.

While k-fold Cross-Validation reduces the bias of improper performance estimation, it is a computationally expensive process. Moreover, while $k$-fold Cross-Validation supposedly uses all of the data for training, each fold only trains on part of the data. Our sampling method can be used in conjunction with $k$-fold Cross-Validation, where in each fold we allow sampling of validation examples during training. This coupling can decrease bias while improving test performance. As a matter of fact, our sampling method readily achieves the $k$-fold Cross-Validation objective, as it reduces the bias of improper performance estimation. Therefore, our method can be similarly used as an alternative to $k$-fold Cross-Validation, with the benefit of increased test performance.

\textbf{Bootstrapping. }
Bootstrapping \cite{efron1982jackknife} \textit{for model selection} is a technique which involves a relatively simple procedure where properties of an estimator are estimated by sampling from an approximating distribution. Given a dataset of size $n$, a bootstrap sample is created by sampling $n$ instances uniformly from the data (with replacement). \cmnt{Each example is seen with probability $\left( 1- \frac{1}{n}\right)^n$. That is, for $n$ large enough, each example is seen with probability $\sim \frac{1}{e}$.} The accuracy estimate is then derived by using the bootstrap sample for training and the rest of the instances for testing. Such a procedure allows for a data split of approximately $37\%$ for validation. It has been shown that the bootstrap procedure works well under stability assumptions of the used algorithm (i.e., the ``bootstrap world" closely approximates the real world) \cite{kohavi1995study}.

While the bootstrapping technique effectively reuses the dataset examples, it maintains a definite segregation between train and validation examples. The enlarged training set is due to duplicate examples, and does not use new information for training. Furthermore, the number of instances sampled in the bootstrap procedure is not interchangeable with $p$, the sample probability of our method. While sampling more instances in the bootstrap method would effectively create a larger training set, the validation set's size will necessarily decrease. Contrary to this, our method does not reduce the size of the validation set. Nevertheless, one could conceive of a bootstrap procedure which would essentially achieve the same type of objective as ours. We leave this as a future direction for future work.

\textbf{Thresholdout. }
The practice of data analysis is an adaptive process. In practice, new hypotheses are proposed after seeing the results of previous ones, models are chosen and tuned on the basis of the obtained results, and datasets are shared and reused. This adaptive process can be harmful, in particular, when subject to model selection \cite{dwork2015reusable}.
Threshholdout (Reusable Holdout Method) \cite{dwork2017guilt} is a method of coping with adaptivity by providing means of reliably verifying the results of an arbitrary adaptive data analysis process. 
As in standard analysis, data is split into a training set and a validation set. The algorithm ensures that the validation set maintains the essential guarantees of ``fresh data" over the course of many estimation steps by accessing the validation set only through a suitable differentially private algorithm. \cmnt{More specifically, the algorithm introduces noise to validation results when the difference between train and validation results passes a given threshold.}

Thresholdout is an important tool for increasing the reliability of validation data in an adaptive process. Our method, on the other hand, considers a non-adaptive setting where a model is selected in a single stage. It can thereby be integrated with Thresholdout to ensure the validity of the model selection in an adaptive manner while using part of the validation data in each step of the training procedure.

\textbf{Train after Test. }
Practitioners often follow a common approach of using pure validation data for model selection and then training on the full dataset (including the validation set examples) once a model has been chosen. Such a process is expected to introduce an optimistic bias into the performance estimates \cite{cawley2010over}. For example, when minimizing the cross-validation performance, if the cross-validation statistic has a non-zero variance there is the possibility of over-fitting the model selection criterion. Figure \ref{fig:desnenet_LRM} depicts an example in which this happens when selecting an optimal learning rate multiplier in the DenseNet architecture.

In order to avoid bias in performance estimates, one should potentially consider penalizing the validation statistic. By sampling from the validation, we allow the use of the validation examples while effectively regularizing this effect. In addition, our method contains a range of values, $p \in [0,1]$, which allows for a continuous adjustment of this inherent trade-off. 

\textbf{Data Dredging}. ``Selective Reporting'' (p-hacking, data dredging) occurs when researchers try out several statistical hypotheses and then selectively report those that produce significant results \cite{nuzzo2014scientific, head2015extent}. Some examples include dropping outliers post-analysis, reporting only those variables which produce significant results, presenting results on specific data which yields significant results and adapting analysis to improve test results. More related to our work is an approach in which the training set size is increased until the best result is reached. 

Our method departs from these practices in two factors. First, our method is not adaptive. That is, we do not use previous validation selections to select further models. Second, we do not use test performance scores in our selection process. Our choice of $p$ relies solely on validation scores, with the choice $p$ in the knee region of the validation plot, see Figure \ref{fig:desnenet_LRM}.

\section{Conclusion}
In this paper, we demonstrated the effectiveness of using the validation set for training. While the immoderate use of the validation set for training may potentially introduce a severe form of selection bias, our sampling procedure allows for a continuous, controlled balance between performance and model selection bias. Furthermore, we have shown that stable algorithms are less prone to overfit the model selection process, with experiments showing that even severe sampling does not bias selection. We provide a per-batch sampling procedure for online algorithms and offer a simple heuristic for choosing a value of $p$ at the knee of the validation plot, a point for which model selection bias and performance are balanced. Finally, our results on DenseNet suggest that our sampling procedure can be used in state-of-the-art deep networks, achieving significantly improved performance.

\bibliography{bibfile}

\begin{thebibliography}{26}
\providecommand{\natexlab}[1]{#1}
\providecommand{\url}[1]{\texttt{#1}}
\expandafter\ifx\csname urlstyle\endcsname\relax
  \providecommand{\doi}[1]{doi: #1}\else
  \providecommand{\doi}{doi: \begingroup \urlstyle{rm}\Url}\fi

\bibitem[Abadi et~al.(2016)Abadi, Barham, Chen, Chen, Davis, Dean, Devin,
  Ghemawat, Irving, Isard, et~al.]{abadi2016tensorflow}
Abadi, Mart{\'\i}n, Barham, Paul, Chen, Jianmin, Chen, Zhifeng, Davis, Andy,
  Dean, Jeffrey, Devin, Matthieu, Ghemawat, Sanjay, Irving, Geoffrey, Isard,
  Michael, et~al.
\newblock Tensorflow: A system for large-scale machine learning.
\newblock In \emph{OSDI}, volume~16, pp.\  265--283, 2016.

\bibitem[Altman(1992)]{altman1992introduction}
Altman, Naomi~S.
\newblock An introduction to kernel and nearest-neighbor nonparametric
  regression.
\newblock \emph{The American Statistician}, 46\penalty0 (3):\penalty0 175--185,
  1992.

\bibitem[Bousquet \& Elisseeff(2002)Bousquet and
  Elisseeff]{bousquet2002stability}
Bousquet, Olivier and Elisseeff, Andr{\'e}.
\newblock Stability and generalization.
\newblock \emph{Journal of machine learning research}, 2\penalty0
  (Mar):\penalty0 499--526, 2002.

\bibitem[Cawley \& Talbot(2010)Cawley and Talbot]{cawley2010over}
Cawley, Gavin~C and Talbot, Nicola~LC.
\newblock On over-fitting in model selection and subsequent selection bias in
  performance evaluation.
\newblock \emph{Journal of Machine Learning Research}, 11\penalty0
  (Jul):\penalty0 2079--2107, 2010.

\bibitem[Devroye \& Wagner(1979)Devroye and Wagner]{devroye1979distribution}
Devroye, Luc and Wagner, Terry.
\newblock Distribution-free performance bounds for potential function rules.
\newblock \emph{IEEE Transactions on Information Theory}, 25\penalty0
  (5):\penalty0 601--604, 1979.

\bibitem[Dwork et~al.(2015)Dwork, Feldman, Hardt, Pitassi, Reingold, and
  Roth]{dwork2015reusable}
Dwork, Cynthia, Feldman, Vitaly, Hardt, Moritz, Pitassi, Toniann, Reingold,
  Omer, and Roth, Aaron.
\newblock The reusable holdout: Preserving validity in adaptive data analysis.
\newblock \emph{Science}, 349\penalty0 (6248):\penalty0 636--638, 2015.

\bibitem[Dwork et~al.(2017)Dwork, Feldman, Hardt, Pitassi, Reingold, and
  Roth]{dwork2017guilt}
Dwork, Cynthia, Feldman, Vitaly, Hardt, Moritz, Pitassi, Toniann, Reingold,
  Omer, and Roth, Aaron.
\newblock Guilt-free data reuse.
\newblock \emph{Communications of the ACM}, 60\penalty0 (4):\penalty0 86--93,
  2017.

\bibitem[e~Elissee(2000)]{e2000study}
e~Elissee, Andr.
\newblock A study about algorithmic stability and their relation to
  generalization performances.
\newblock 2000.

\bibitem[Efron(1982)]{efron1982jackknife}
Efron, Bradley.
\newblock \emph{The jackknife, the bootstrap, and other resampling plans},
  volume~38.
\newblock Siam, 1982.

\bibitem[Elisseeff et~al.(2003)Elisseeff, Pontil, et~al.]{elisseeff2003leave}
Elisseeff, Andr{\'e}, Pontil, Massimiliano, et~al.
\newblock Leave-one-out error and stability of learning algorithms with
  applications.
\newblock \emph{NATO science series sub series iii computer and systems
  sciences}, 190:\penalty0 111--130, 2003.

\bibitem[Freedman \& Freedman(1983)Freedman and Freedman]{freedman1983note}
Freedman, David~A and Freedman, David~A.
\newblock A note on screening regression equations.
\newblock \emph{the american statistician}, 37\penalty0 (2):\penalty0 152--155,
  1983.

\bibitem[Geisser(1975)]{geisser1975predictive}
Geisser, Seymour.
\newblock The predictive sample reuse method with applications.
\newblock \emph{Journal of the American statistical Association}, 70\penalty0
  (350):\penalty0 320--328, 1975.

\bibitem[Head et~al.(2015)Head, Holman, Lanfear, Kahn, and
  Jennions]{head2015extent}
Head, Megan~L, Holman, Luke, Lanfear, Rob, Kahn, Andrew~T, and Jennions,
  Michael~D.
\newblock The extent and consequences of p-hacking in science.
\newblock \emph{PLoS biology}, 13\penalty0 (3):\penalty0 e1002106, 2015.

\bibitem[Hearst et~al.(1998)Hearst, Dumais, Osuna, Platt, and
  Scholkopf]{hearst1998support}
Hearst, Marti~A., Dumais, Susan~T, Osuna, Edgar, Platt, John, and Scholkopf,
  Bernhard.
\newblock Support vector machines.
\newblock \emph{IEEE Intelligent Systems and their applications}, 13\penalty0
  (4):\penalty0 18--28, 1998.

\bibitem[Huang et~al.(2016)Huang, Liu, Weinberger, and van~der
  Maaten]{huang2016densely}
Huang, Gao, Liu, Zhuang, Weinberger, Kilian~Q, and van~der Maaten, Laurens.
\newblock Densely connected convolutional networks.
\newblock \emph{arXiv preprint arXiv:1608.06993}, 2016.

\bibitem[Kohavi et~al.(1995)]{kohavi1995study}
Kohavi, Ron et~al.
\newblock A study of cross-validation and bootstrap for accuracy estimation and
  model selection.
\newblock In \emph{Ijcai}, volume~14, pp.\  1137--1145. Montreal, Canada, 1995.

\bibitem[Krizhevsky \& Hinton(2009)Krizhevsky and
  Hinton]{krizhevsky2009learning}
Krizhevsky, Alex and Hinton, Geoffrey.
\newblock Learning multiple layers of features from tiny images.
\newblock 2009.

\bibitem[LeCun(1998)]{lecun1998mnist}
LeCun, Yann.
\newblock The mnist database of handwritten digits.
\newblock \emph{http://yann. lecun. com/exdb/mnist/}, 1998.

\bibitem[Mukherjee et~al.(2006)Mukherjee, Niyogi, Poggio, and
  Rifkin]{mukherjee2006learning}
Mukherjee, Sayan, Niyogi, Partha, Poggio, Tomaso, and Rifkin, Ryan.
\newblock Learning theory: stability is sufficient for generalization and
  necessary and sufficient for consistency of empirical risk minimization.
\newblock \emph{Advances in Computational Mathematics}, 25\penalty0
  (1):\penalty0 161--193, 2006.

\bibitem[Nuzzo(2014)]{nuzzo2014scientific}
Nuzzo, Regina.
\newblock Scientific method: statistical errors.
\newblock \emph{Nature News}, 506\penalty0 (7487):\penalty0 150, 2014.

\bibitem[Quionero-Candela et~al.(2009)Quionero-Candela, Sugiyama, Schwaighofer,
  and Lawrence]{quionero2009dataset}
Quionero-Candela, Joaquin, Sugiyama, Masashi, Schwaighofer, Anton, and
  Lawrence, Neil~D.
\newblock \emph{Dataset shift in machine learning}.
\newblock The MIT Press, 2009.

\bibitem[Refaeilzadeh et~al.(2009)Refaeilzadeh, Tang, and
  Liu]{refaeilzadeh2009cross}
Refaeilzadeh, Payam, Tang, Lei, and Liu, Huan.
\newblock Cross-validation.
\newblock In \emph{Encyclopedia of database systems}, pp.\  532--538. Springer,
  2009.

\bibitem[Schaffer(1993)]{schaffer1993selecting}
Schaffer, Cullen.
\newblock Selecting a classification method by cross-validation.
\newblock \emph{Machine Learning}, 13\penalty0 (1):\penalty0 135--143, 1993.

\bibitem[Shalev-Shwartz \& Ben-David(2014)Shalev-Shwartz and
  Ben-David]{shalev2014understanding}
Shalev-Shwartz, Shai and Ben-David, Shai.
\newblock \emph{Understanding machine learning: From theory to algorithms}.
\newblock Cambridge university press, 2014.

\bibitem[Shalev-Shwartz et~al.(2009)Shalev-Shwartz, Shamir, Sridharan, and
  Srebro]{shalev2009learnability}
Shalev-Shwartz, Shai, Shamir, Ohad, Sridharan, Karthik, and Srebro, Nathan.
\newblock Learnability and stability in the general learning setting.
\newblock 2009.

\bibitem[Stone(1974)]{stone1974cross}
Stone, Mervyn.
\newblock Cross-validatory choice and assessment of statistical predictions.
\newblock \emph{Journal of the royal statistical society. Series B
  (Methodological)}, pp.\  111--147, 1974.

\end{thebibliography}
\bibliographystyle{icml2018}

\newpage
\section{Appendix}

\subsection{Proof of Main Theorem}

\begin{proof}
We start by proving two standard technical lemmas. 
\begin{lemma}
\label{lemma:condition1}
Suppose OAROS stability holds with $\epsilon_1(n)$. Then
\begin{equation*}
\abs{ \E_{T} \E_{y,y' \sim \D} \left[ \ell(\A(T \cup \{ y'\}), y) - \ell(\A(T \cup \{ y\}), y) \right]} \leq \epsilon_1(n)
\end{equation*}
\end{lemma}
\begin{proof}
We have that
\begin{align*}
&\E_{S} \E_{y,y' \sim \D} \left[ \ell(\A(S \cup \{ y'\}), y) - \ell(\A(S \cup \{ y\}), y)\right] = \\
&\E_{\tilde{S}} \left[ L_{\D}(\A(\tilde{S})) - L_{\tilde{S}}(\A(\tilde{S})) \right],
\end{align*}
Then it is enough to show that 
\begin{align*}
&\E_S \E_{x'\sim \D}\E_{z \sim U(S)} \left[ \ell(\A(S \cup \{x'\} \backslash \{z\}), z) - \ell(\A(S), z) \right] = \\
&\E_{S} \left[ L_{\D}(\A(S)) - L_{S}(\A(S)) \right],
\end{align*}
then we are done by definition of OAROS stability and Jensen's inequality. \\
Denote $S = \left\{ z_1, \hdots z_n \right\}$. Since $S$ and $x'$ are drawn i.i.d from $\D$ we have
\begin{align*}
&\E_{S} L_{\D}(\A(S)) = \\
&\E_{S}\E_{x' \sim \D} \ell(\A(S), x') = \\
&\E_{S}\E_{x' \sim \D} \ell(\A(S \cup \{x'\} \backslash\{z_i\} ), z_i) = \\
&\E_{S}\E_{x' \sim \D} \E_{z \sim U(S)} \ell(\A(S \cup \{x'\} \backslash \{z\}), z).
\end{align*}
On the other hand,
\begin{align*}
&\E_{S} L_{S}(\A(S)) = \\
&\E_{S}\E_{z \sim U(S)} \ell(\A(S), z) = \\
&\E_{S}\E_{x' \sim \D} \E_{z \sim U(S)} \ell(\A(S), z).
\end{align*}
The proof is complete since $\epsilon(n+1) \leq \epsilon(n)$.
\end{proof}
\begin{lemma}
\label{lemma:finalCondition}
Suppose OAROS Stability holds with $\epsilon_1(n)$ and also
\begin{align}
\nonumber
\E_{T}\E_{x'\sim \D}[ &\E_{z\sim U(T)}  \ell(\A(T \cup \{x'\} \backslash \{z\}, z) \\&- E _{y \sim \D}\ell(\A(T \cup \{x'\}), y)) ] 
\leq \tilde{\epsilon}(n).
\label{eq:lemmaFinalConditionAssumption}
\end{align}
Then
\begin{align*}
\E_{T,V} \E_{y,y' \sim U(V)} \left[ \ell(\A(T), y) - \ell(\A(T \cup \{ y'\}), y)\right] \\
\leq \pth{1 + \frac{1}{m}}\epsilon_1(n) + \tilde{\epsilon}(n).
\end{align*}
\end{lemma}
\begin{proof}
We have that
\begin{align*}
&\E_{T,V} \E_{y,y' \sim U(V)} \left[ \ell(\A(T), y) - \ell(\A(T \cup \{ y'\}), y)\right] = \\
&\E_{T} \E_{y,y' \sim \D} \left[ \ell(\A(T), y) - (1-\frac{1}{m})\ell(\A(T \cup \{ y'\}), y) - \right. \\
&~\qquad \qquad \left. \frac{1}{m} \ell(\A(T \cup \{ y\}), y) \right] \leq \\
& \E_{T} \E_{y,y' \sim \D} \left[ \ell(\A(T), y) - (1-\frac{1}{m})\ell(\A(T \cup \{ y'\}), y) + \right.\\
&~\qquad \qquad \left. \frac{1}{m} \pth{\ell(\A(T \cup \{ y'\}), y) + \epsilon_1(n)} \right] = \\
& \E_{T} \E_{y,y' \sim \D} \left[ \ell(\A(T), y) - \ell(\A(T \cup \{ y'\}), y) + \frac{\epsilon_1(n)}{m} \right] = \\
& \E_{T} \E_{y' \sim \D} \bigg[ \E_{y \sim \D}\ell(\A(T), y) - \E_{z\sim U(T)}  \ell(\A(T \cup \{x'\} \backslash \{z\}, z) + \\
&~\qquad \qquad \E_{z\sim U(T)}  \ell(\A(T \cup \{x'\} \backslash \{z\}, z) - \E_{y \sim \D}\ell(\A(T \cup \{ y'\}), y) + \\
&~\qquad \qquad \frac{\epsilon_1(n)}{m} \bigg] = \\
& \E_{T} \E_{y' \sim \D} \bigg[ \E_{z \sim U(T)}\ell(\A(T), z) - \E_{z\sim U(T)}  \ell(\A(T \cup \{x'\} \backslash \{z\}, z) + \\
&~\qquad \qquad \E_{z\sim U(T)}  \ell(\A(T \cup \{x'\} \backslash \{z\}, z) - \E_{y \sim \D}\ell(\A(T \cup \{ y'\}), y) + \\
&~\qquad \qquad \frac{\epsilon_1(n)}{m} \bigg] \leq \\
&\epsilon_1(n) + \tilde{\epsilon}(n) + \frac{\epsilon_1(n)}{m}.
\end{align*}
In the first inequality we used the result of Lemma \ref{lemma:condition1}, and in the final inequality we used OAROS stability and the assumption in Equation (\ref{eq:lemmaFinalConditionAssumption}).
\end{proof}

Continuing with the proof of the theorem, we restate the theorem for the ERM from the Preliminaries Section:
\begin{theorem}
\label{thm:stabilityNoOverfit}
Let $\A$ be a learning algorithm with OAROS stability rate $\epsilon(n)$, then
\begin{equation*}
\E_{S \sim \D^m} \left[ L_{\D}(\A(S)) - L_S(\A(S)) \right] \leq \epsilon(n).
\end{equation*}
\end{theorem}

\noindent We continue with the proof of the theorem.
By Equation (\ref{eq:ERM_convergence}) it holds that
\begin{equation*}
\E_{\tilde{T}} \left[ L_T(\A(T)) - L_{\tilde{T}}(\A(\tilde{T})) \right] \leq \epsilon_2(n).
\end{equation*}
Then, using Theorem \ref{thm:stabilityNoOverfit},
\begin{equation*}
\E_{\tilde{T}} \left[ L_T(\A(T)) - L_{\D}(\A(\tilde{T})) \right] - \epsilon_1(n) \leq \epsilon_2(n),
\end{equation*}
We now have that 
\begin{align*}
&\epsilon_1(n) + \epsilon_2(n) \geq \E_{\tilde{T}} \left[ L_T(\A(T)) - L_\D(\tilde{T})) \right] = \\
& \E_{T}\E_{x'\sim \D} \left[ L_T(\A(T)) - \E_{x'\sim \D}L_\D(\A(T \cup \{x'\})) \right] = \\
& \E_{T}\E_{x'\sim \D}\E_{z\sim U(T)} \ell(\A(T), z) - \\
&\E_{T}\E_{x'\sim \D}E _{y \sim \D}\ell(\A(T \cup \{x'\}), y).
\end{align*}
By OAROS Stability this means
\begin{align*}
\E_{T}\E_{x'\sim \D}[ &\E_{z\sim U(T)}  \ell(\A(T \cup \{x'\} \backslash \{z\}, z) \\&- E _{y \sim \D}\ell(\A(T \cup \{x'\}), y)) ] \leq 2\epsilon_1(n) + \epsilon_2(n)
\end{align*}
Then, by Lemma \ref{lemma:finalCondition}
\begin{align*}
\E_{T,V} \E_{y,y' \sim U(V)} \left[ \ell(\A(T), y) - \ell(\A(S \cup \{ y'\}), y)\right] \\ \leq \pth{1 + \frac{1}{m}}\epsilon_1(n) + 2\epsilon_1(n) + \epsilon_2(n),
\end{align*}
and the proof is complete.
\end{proof}



\subsection{Validation Stability Bound}
Suppose our training set $T$ and validation set $V$ are both sampled from the same distribution $\D$ with $\left| V \right| = m, \left| S \right| = n$. Furthermore, let $z' \sim U(V)$. To show that we can sample from the validation set without over-fitting, we want to show that with probability at least $1-\delta$:
%
\begin{equation}
\label{eq:FinalObjective}
\left| L_{ V } (\A (S \cup \{ z' \}) ) -  L_{ V } (\A (S) ) \right| \leq \frac{\epsilon (n,m)}{\delta}
\end{equation}
Then, if the validation loss doesn't change much, the order in which we pick the best model won't change as well. \\
We consider a weaker formulation to Equation (\ref{eq:FinalObjective}):
\begin{equation}
\label{eq:WeakObjective}
\E_{S,V} \E_{i,j \sim U(m)} \left| \ell(\A(S \cup \{ z_i\}), z_j) - \ell(\A(S), z_j)\right| \leq \epsilon(n,m)
\end{equation}
Let's see that Equation (\ref{eq:FinalObjective}) follows immediately from (\ref{eq:WeakObjective}) using Markov inequality. From (\ref{eq:WeakObjective}) and Jensen's inequality we have that
\begin{align*}
\epsilon(n,m) &\geq\E_{S,V}\E_{i \sim U(m)}  \\
& \quad \left| \E_{j \sim U(m)}  \ell(\A(S \cup \{ z_i\}), z_j) - \ell(\A(S), z_j)\right| = \\
& \E_{S,V}\E_{i \sim U(m)}  \left|    L_{ V } (\A (S \cup \{ z' \}) ) -  L_{ V } (\A (S) ) \right|.
\end{align*}
And recall that using Markov's inequality for R.V. ${X \geq 0}$ with $EX \leq M$ we have that with probability at least $1-\delta$, $X \leq \frac{M}{\delta}$. This gives us (\ref{eq:FinalObjective}) for ${X = \left| L_{ V } (\A (S \cup \{ z' \}) ) -  L_{ V } (\A (S) )\right|}$.

\subsection{Linear Regression Results}

We fit a 1-dimensional cubic function using noisy sampled examples and polynomial features. The dataset of 15 sampled points in total was divided to 10 examples for training and 5 examples for validation (sampled from the noisy distribution). The validation set was used to select the optimal order of polynomial to fit, with values 1, 2, 3. Similar to the the k-NN experiment, we tested the effect of biased validation and test sets on our sampling procedure by introducing artificial bias to the validation and test sets separately. The biased functions are drawn in Figure \ref{fig:reg_bias_plots}. 

\begin{figure}[h!]
\centering
\includegraphics[width=\linewidth]{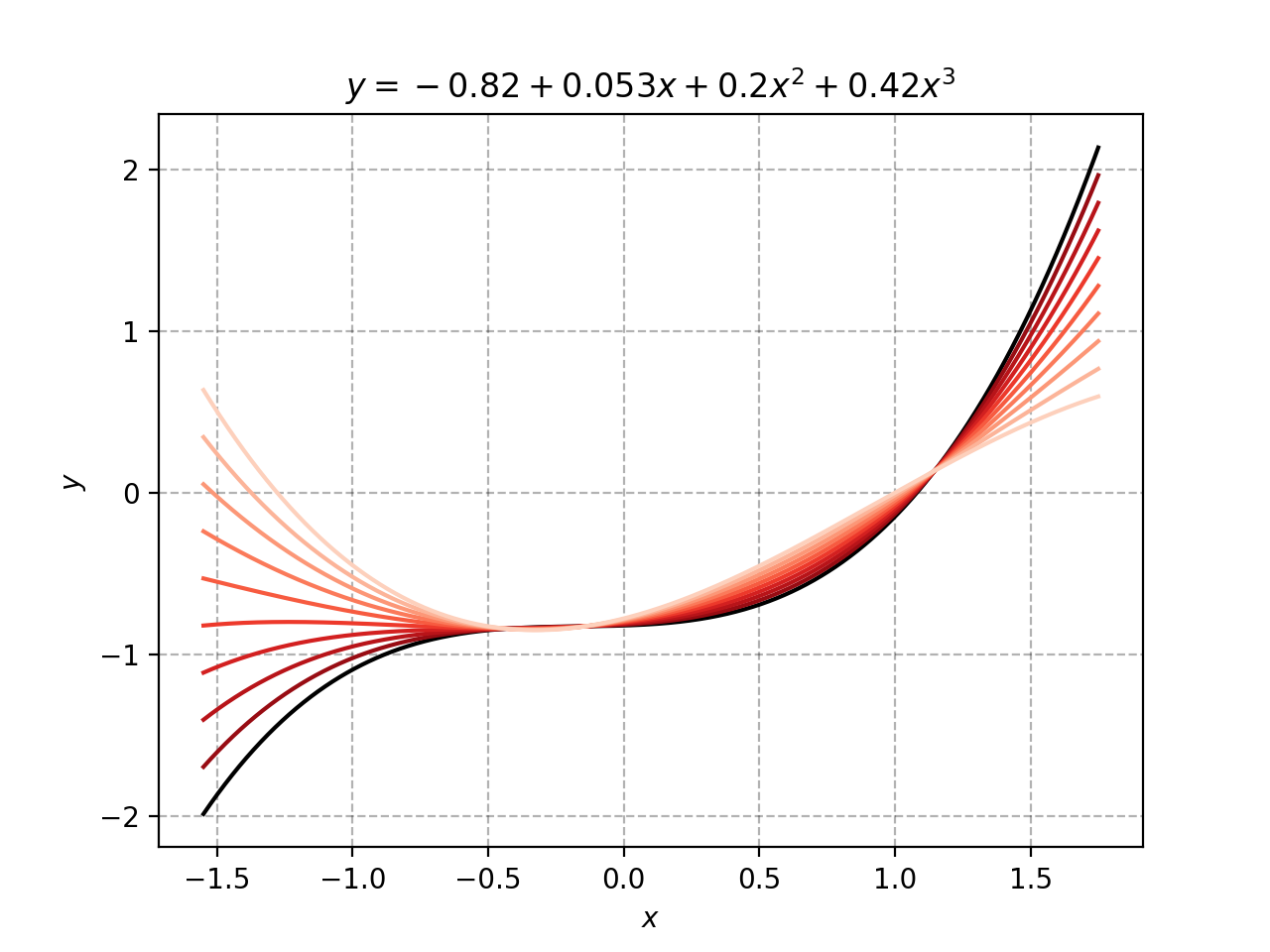}
\caption{Plot shows function used for regression. Black line shows line used for sampling training set. Red lines show biased lines used for creating biased validation and test sets.}
\label{fig:reg_bias_plots}
\end{figure}

\begin{figure}[h!]
\centering
\includegraphics[width=\linewidth]{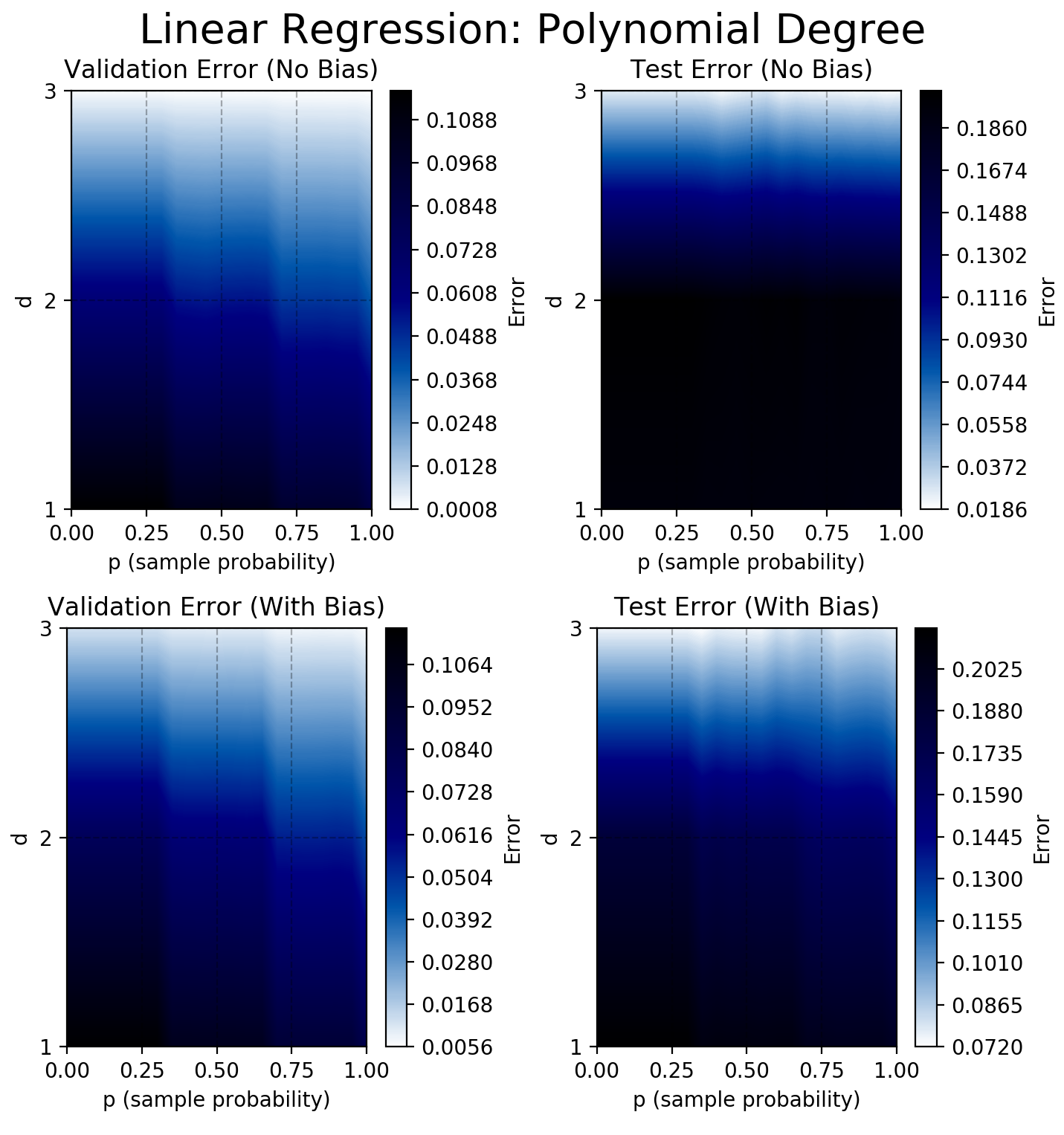}
\caption{Validation and test accuracies for the presample procedure using linear regression. Top graphs show results for unbiased validation and test sets, while bottom graphs show results for biased validation and test sets. }
\label{fig:linear_reg}
\end{figure}

Figure \ref{fig:linear_reg} shows mean error rates of the validation and test sets for different choices of polynomial degrees. Colors of the contour map range from black to blue to white, with black depicting lower accuracy scores, and white depicting highest accuracy. The x-axis of the map shows different $p$ values with $p$ ranging from $0$ to $1$. Every point on the grid shows a result of an experiment with a given $p$ and degree $d$. The plots show that an increase in $p$ brought to a uniform decrease in validation error, thereby allowing for a correct choice of model (even when the whole validation set was used for training). Error rates monotonically decrease in accordance to the sampling probability. These become most significant in the case of biased validation and test sets, where sampling is essential for the training process.


\end{document}